\documentclass[letterpaper]{article} 
\usepackage{aaai24}  
\usepackage{times}  
\usepackage{helvet}  
\usepackage{courier}  
\usepackage[hyphens]{url}  
\usepackage{graphicx} 
\usepackage{natbib}  
\usepackage{caption} 
\usepackage{amsthm}
\usepackage{thm-restate}
\usepackage{dsfont}
\usepackage{algorithm}
\usepackage{algorithmic}
\usepackage{amsmath}
\usepackage{booktabs}
\usepackage{subcaption}

\usepackage{placeins}
\usepackage{multirow}

\usepackage{makecell}

\newcommand{\mcD}{\mathcal{D}}
\newcommand{\calD}{\mathcal{D}}
\newcommand{\mcU}{\mathcal{U}}
\newcommand{\argmin}{\mathop{\mathrm{argmin}}}

\newcommand{\bcm}{\mathrm{BCM}}
\newcommand{\mix}{\mathrm{mix}}
\newcommand{\prd}{\mathrm{prod}}

\usepackage{newfloat}
\usepackage{listings}
\DeclareCaptionStyle{ruled}{labelfont=normalfont,labelsep=colon,strut=off} 
\floatstyle{ruled}
\newfloat{listing}{tb}{lst}{}
\floatname{listing}{Listing}
\title{Calibrated One Round Federated Learning with Bayesian Inference in the Predictive Space}
\author {   
    Mohsin Hasan\textsuperscript{\rm 1,2},
    Guojun Zhang\textsuperscript{\rm 3},
    Kaiyang Guo\textsuperscript{\rm 3},
    Xi Chen\textsuperscript{\rm 3},
    Pascal Poupart\textsuperscript{\rm 1,2}
}
\affiliations{
    \textsuperscript{\rm 1}University of Waterloo\\
    \textsuperscript{\rm 2}Vector Institute \\
    \textsuperscript{\rm 3}Huawei Noah's Ark Lab\\
    mohsin.hasan@uwaterloo.ca, \{guojun.zhang, guo.kaiyang, xi.chen4\}@huawei.com, ppoupart@uwaterloo.ca
}

\usepackage{bibentry}

\begin{document}

\maketitle
\begin{abstract}

Federated Learning (FL) involves training a model over a dataset distributed among clients, with the constraint that each client’s dataset is localized and possibly heterogeneous. In FL, small and noisy datasets are common, highlighting the need for well-calibrated models that represent the uncertainty of predictions. 
The closest FL techniques to achieving such goals are the Bayesian FL methods which collect parameter samples from local posteriors, and aggregate them to approximate the global posterior.
To improve scalability for larger models, one common Bayesian approach is to approximate the global predictive posterior by multiplying local predictive posteriors. In this work, we demonstrate that this method gives \emph{systematically overconfident} predictions, and we remedy this by proposing \emph{$\beta$-Predictive Bayes}, a Bayesian FL algorithm that interpolates between a mixture and product of the predictive posteriors, using a tunable parameter $\beta$. This parameter is tuned to improve the global ensemble’s calibration, before it is distilled to a single model. Our method is evaluated on a variety of regression and classification datasets to demonstrate its superiority in calibration to other baselines, even as data heterogeneity increases. Code available at  \url{https://github.com/hasanmohsin/betaPredBayesFL}. 
\end{abstract}

\section{Introduction}

Federated learning (FL) is a machine learning paradigm that trains a statistical model using decentralized data stored on client devices, with the constraint that client data is kept local~\cite{mcmahan17a}. FL has found use in smartphone applications, as well as fields such as healthcare and finance due to the abundance of use-cases where sensitive training data is owned by separate entities \cite{zheng2022applications}. 

The workflow of a typical FL algorithm involves the local training of models on each client, followed by the communication and aggregation of these into a single global model on a central server. Many FL techniques alternate between these two steps until some notion of convergence is reached.

For the purpose of this work, we are concerned with the design of FL algorithms with three key metrics in mind:

\begin{enumerate}
    \item \textbf{Communication Cost}: the transmission of models between clients and servers can be expensive, especially when each model has numerous parameters (e.g.~neural networks). Cost effective FL techniques therefore aim to maximize the amount of local computation at each client, while reducing the rounds of communication. 

    \item \textbf{Performance with Heterogeneous Data}: clients may have different data distributions from each other. This causes the local client models to drift apart, presenting issues when aggregating them into a global model \cite{zhao2018federated}. In particular, when client datasets are heavily heterogeneous, global models tend to perform poorly on clients' local datasets.  
    
    \item \textbf{Calibration}: Clients' data may have too few training points, or too much variance. Therefore, a valuable goal in FL is to produce models that are \textbf{well-calibrated}, or in other words: make probabilistic predictions with accurate uncertainty estimates.
\end{enumerate}

Many techniques frame FL as a distributed optimization problem. For these, there is typically a trade-off between achieving a low communication cost and good performance on heterogeneous data: they improve global performance at the cost of more communication ~\cite{mcmahan17a, fedAvgDivHetero}. Moreover, they do not have any systematic way of calibrating the global model.

To alleviate these problems, an alternative branch of FL techniques takes a Bayesian perspective on learning the model. They approximate the Bayesian posterior distribution of each client model, and aggregate them into the global Bayesian posterior~\cite{al-shedivat2021federated, ep_mcmc}. This allows the training of more effective global models on heterogeneous data by leveraging client uncertainty during aggregation ~\cite{al-shedivat2021federated}. Certain Bayesian FL methods also demonstrate how to aggregate local posteriors in only a single round of communication~\cite{ep_mcmc}. Furthermore, these methods explicitly represent and adjust the uncertainty in model parameters, i.e., they should be well-calibrated in principle.

In these methods, the Bayesian posterior is a distribution over model parameters, and it is expensive to represent and manipulate. Therefore, it becomes necessary to apply approximations to the posterior. Many methods work with crude approximations to the client posterior (as e.g., a multivariate Gaussian)~\cite{al-shedivat2021federated, ep_mcmc}, which can incur heavy error, resulting in poor calibration and accuracy.

A promising Bayesian method for FL is the Bayesian Committee Machine (BCM, \citealt{tresp2000bcm}), which instead aggregates the distribution over \emph{model predictions} (referred to as the \textbf{Bayesian predictive posterior}) rather than the posterior over \emph{parameters}. The former allows for more accurate approximation due to its lower dimensionality. Nevertheless, these methods need to rely on an approximate aggregation technique, which can add bias to the global model.

In this work, we argue that the BCM is poorly calibrated due to producing overconfident predictions. We remedy this by proposing a new aggregation method that interpolates between the BCM predictions, and those made by a mixture model over the local predictive posteriors. This server aggregation results in an ensemble model, which is then distilled into a single model, to be sent back to clients. 

The primary contributions of our work are as follows:

\begin{itemize}
    \item We propose a novel algorithm for Bayesian FL, called \emph{$\beta$-Predictive Bayes}. This method operates in \emph{a single round of communication}, while benefiting from performance over heterogeneous data like the BCM. On the other hand, it does not suffer from the same calibration issues as the BCM.
    \item We empirically evaluate $\beta$-Predictive Bayes on multiple regression and classification datasets, using partitions simulating varying levels of data heterogeneity. The proposed technique competes with or outperforms other baselines with respect to calibration, particularly when data heterogeneity increases.
\end{itemize}

\section{Background}

In federated learning (FL), data is distributed across several clients.  Let $\mathcal{D}=\mathcal{D}_1 \cup ... \cup \mathcal{D}_n$ where $\mathcal{D}_i = \{(x_1,y_1),...,(x_{k_i},y_{k_i})\}$ is the dataset of size $k_i$ at client $i$, consisting of inputs $x$ and targets $y$. The goal is to learn a predictive model without any data leaving each client to preserve privacy. Let $m_{\theta}$ denote a model, parameterized by weights $\theta$, which outputs a predictive distribution $m_{\theta}(x) = p(y|x,\theta)$. In a typical FL algorithm, each client learns a local model $m_{\theta_i}$, which it shares with a trusted server that aggregates them into a global model $m_{\bar{\theta}}$. For example, Federated Averaging (FedAvg)~\citep{mcmahan17a} aggregates local models by taking the average of their parameters (i.e., $\bar{\theta} = \sum_i \theta_i k_i/(\sum_i k_i)$).

In practice, the datasets are often heterogeneous, which means that they are sampled from different distributions. To deal with heterogeneity and avoid client divergence, FedAvg~\citep{mcmahan17a} and many other variants~\citep{mohri2019agnostic, wang2019federated,  li2020fedprox, wang2020tackling, li2020fedbn} perform frequent rounds of model updates and averaging where in each round the clients update their local models based on a few steps of gradient descent, or more generally, some form of partial training. This can be quite costly due to the increased communication and the need for synchronization at each round. Some methods alter the training to improve performance in heterogeneous settings. For instance, FedProx \citep{li2020fedprox} adds a penalty between the global model and each client's local model to the local training loss, while \citet{reddi2021adaptive} use an adaptive optimizer (such as Adam \citep{adam}) at the server for aggregation.

\subsection{Bayesian Learning} 
Bayesian learning is a training method that takes into account uncertainty over parameters \citep{bishop_ml}. It operates by setting a ``model space prior" $p(\theta)$, which encapsulates beliefs about the model parameters before observing data. Upon processing the dataset, this prior is updated to the ``model space posterior" using Bayes' Rule: $p(\theta|\mcD) \propto p(\theta)p(\mcD| \theta)$. This posterior is then used to make predictions. 

One method for doing so is to obtain samples from the posterior, $S = \{\theta_1,...,\theta_M\} \sim p(\theta|\mcD)$ through a method such as \emph{Markov Chain Monte Carlo (MCMC)}, and use them to approximate the ``posterior predictive distribution":
\begin{align}
    p(y|x,\mcD) = \int p(y|x, \theta) p(\theta|\mcD) d\theta \approx 
    \frac{1}{M}\sum_j p(y|x, \theta_j).
    \nonumber
\end{align}
\subsection{Knowledge Distillation } The goal of knowledge distillation is to compress a given larger ``teacher" model $m_{\theta_T}(x)$ into a smaller ``student" model $m_{\theta_S}(x)$ that matches its predictions on a shared data distribution \citep{Hinton2015DistillingTK}.

A number of FL techniques assume the server has access to a public (unlabelled) dataset $\mcU$ that serves as the distillation dataset \citep{guha2019, practicalOneShot2021, chen2021fedbe}. The student is trained to minimize a loss of the form $\mathcal{L}(\theta_S) = \sum_{x \in \mcU} l(m_{\theta_S}(x), m_{\theta_T}(x))$. Here, $l(\cdot, \cdot)$ measures the discrepancy between predictions. For classification, this discrepancy can be measured by the Kullback-Leibler divergence between the class distributions, while for regression tasks, it can be measured by the mean-squared error.

\section{Related Work}

\subsection{Bayesian Techniques in FL} 
Bayesian learning offers FL techniques the advantage of better performance in heterogeneous settings. As argued by \citet{al-shedivat2021federated}, FedAvg can be thought of as a technique that obtains the mean of the global posterior, if each local posterior is approximated as a Gaussian with the identity as the covariance matrix. Thus FedAvg implicitly assumes a form of homogeneity, which may not be practical. Bayesian techniques offer the ability to remedy this by using more realistic approximations to the local posteriors.

Existing Bayesian FL techniques focus on approximating the global model space posterior $p(\theta|\mcD)$ from the local model space posteriors $p(\theta|\mcD_i)$ ~\citep{ep_mcmc,al-shedivat2021federated}. 

\textbf{Embarrassingly Parallel MCMC} (EP MCMC, \citealt{ep_mcmc}) implements Bayesian inference by drawing MCMC samples from each local posterior (with a corrective factor from the prior), and then estimating the local densities either as Gaussians or with a kernel density estimator. These local densities are then aggregated via multiplication (again with a prior corrective factor) to obtain an approximation for the global model space posterior. This global density is then sampled to obtain the desired posterior samples. It is worth noting that the original work was not designed for use with neural networks, and the memory costs associated with the method make it intractable for this setting. For instance, when approximating the local posteriors as Gaussians and aggregating them, a computational cost of $O(d^3)$ is required for inverting the covariance matrices, where $d$ is the number of neural network parameters. This method is notable however for operating with only a single communication round.

\textbf{Federated Posterior Averaging}~\citep{al-shedivat2021federated} is similar to EP MCMC, except that it approximates the local posteriors as Gaussians, and devises an efficient iterative algorithm for aggregating the local posteriors (with cost linear in the number of parameters). However, the method requires multiple rounds of communication to be effective. 

The main issue with both these techniques is that they require some approximation of the global model space posterior (e.g., in the form of a Gaussian), which can often be inaccurate when the number of model parameters is large. Such approximations are especially poor for neural network models, where the model space posterior is known to be multimodal~\citep{pourzanjani2017improving}. 

The \textbf{Bayesian Committee Machine} (BCM, \citealt{tresp2000bcm, rbcm}) can be thought of as an aggregation method that combines low dimensional predictive posteriors $p(y|x,\mcD_i)$, rather than the parameter posteriors $p(\theta|\mcD_i)$. It aggregates as 
\begin{align}
    \label{eq:prod_rule}
    p(y|x,\mcD) &= \frac{1}{p^{m-1}(y|x)} \prod_{i} p(y|x, \mcD_i). 
\end{align}
This formula is correct assuming that the data shards $\mcD_i$ are independent from each other (and conditionally independent given the single query point $(x,y)$). We note that this can be satisfied if the data shards form clusters in input space, in other words, if the data is heterogeneous in a certain way. 

The BCM results in an ensemble model over the local Bayesian samples. The advantage of this method is that predictive posteriors are much simpler, and lend themselves to, for instance, Gaussian approximations, without sacrificing accuracy. On the other hand, the aggregation formula is no longer exact, but only approximately true, which causes other inaccuracies (specifically, in calibration). The BCM is the starting point for our proposed method. 

The \textbf{Generalized Robust BCM} \citep{liu18a} proposes corrections to the calibration shortcomings of the BCM in regression. Our work analyzes the calibration under different conditions, and extends the analysis to the classification setting. In addition, the proposed correction algorithm requires sharing a subset of the data to all clients, which is incompatible with FL privacy constraints. In this work, we propose a different correction procedure amenable to FL.

\subsection{One-Shot Federated Learning} 
Owing to the importance of efficient communication, several methods have been developed to perform FL within a single communication round. \textbf{One-Shot FL} \citep{guha2019}, as well as \textbf{Federated Learning via Knowledge Transfer (FedKT)} \cite{practicalOneShot2021} perform one round training by constructing an ensemble from the client models, and compressing it into a  single model using knowledge distillation on a public unlabeled dataset. The methods differ in how they construct the teacher ensemble: whereas ``One-Shot FL" averages the client predictions (and was tailored for SVM models), ``FedKT" aggregates based on majority voting by local models (and only applies to classification tasks). The latter technique uses a scheme called ``consistent voting", where it uses discrepancies between client votes to determine which clients are uncertain about their predictions, and thus can be ignored in the majority vote. Our work is similar to these approaches, but derives aggregation rules for the local client models using a Bayesian perspective, and is applicable to any type of tasks or predictive models. 
In contrast to our method, these existing techniques do not prioritize well-calibrated predictions. 

\section{Analysis of BCM Calibration}

As mentioned before, the BCM aggregation requires the independence assumption of the data shards. We analyze when and how this approximation fails: in that it can produce an overconfident global model. 

We can analyze the BCM equation in the context of Gaussian process (GP) regression, since it allows us to explicitly calculate the predictive mean and variance.

We assume a smooth, isotropic kernel function $k(x,x') = k(||x - x'||)$. We assume a model with Gaussian noise, i.e., $y = f(x) + \epsilon$  where $\epsilon \sim \mathcal{N}(0, \sigma^2_o)$ and $\sigma^2_o$ is the ``observation variance" in the predictions (in other words, the aleatoric uncertainty in our predictions). 

In this setting, the BCM predictive posteriors can be approximated as Gaussians $p(y|x,\mcD_i) = \mathcal{N}(\mu_i, \Sigma_i)$ (and with prior $p(y|x) = \mathcal{N}(\mu_p, \Sigma_p)$). The aggregation formula \eqref{eq:prod_rule} now computes a global mean and covariance:
\begin{align}
    \label{eq:var_regr}    \Sigma_g &= \left(\sum_i \Sigma^{-1}_i - (n-1)\Sigma^{-1}_p \right)^{-1}, \\
   \label{eq:mean_regr}   \mu_g &= \Sigma_g \left(\sum_i \Sigma^{-1}_i \mu_i - (n-1)\Sigma^{-1}_p\mu_p\right).
\end{align}
Suppose the input data-points lie in some bounded region $R$, i.e., $x_i \in R$ for all $i$. For the GP we outline two observations about the predictive variance $\sigma^2(x_*)$:
\begin{restatable}[\citealt{choiGPconsistency}]{lemma}{LemmaOne}
\label{lemma:obs_var}
Assume $x^* \in R$. Under some mild conditions on the kernel function, and under the assumption of Gaussian or Laplacian observation noise, as the number of data-points increases $\sigma^2(x^*) \to \sigma^2_o$ (and in addition, the predictive mean converges to the true function value: $\mu(x^*) \to f(x^*)$). 
\end{restatable}

\begin{restatable}[]{lemma}{LemmaTwo}
\label{lemma:prior_var}
Assume $x^* \notin R$, and is sufficiently far away from all training points such that $k(x^*, x_i) \approx 0$ for all $x_i \in R$. Then the predictive variance becomes $\sigma^2(x^*) = \sigma^2_o + k(x^*, x^*)$, which we refer to as the ``prior variance" $\sigma^2_p$. Also, the predictive mean becomes $\mu(x^*) = 0$.
\end{restatable}
    
The proof for this lemma, and all following theorems are in the appendix. Taken together, Lemma 1 and 2 mean that near data, the GP attains the correct mean and variance, and away from it, it reverts to a larger prior variance. This reflects the model's higher uncertainty on unobserved data.

Equipped with these observations, we can analyze the case with partitioned data. We assume two idealized partitions for the overall dataset $\mcD$ into the $m$ shards:

\begin{itemize}
    \item \textbf{Idealized Homogeneous Partition}: each dataset $\mcD_i$ is sampled from the same distribution as global dataset $\mcD$.
    \item \textbf{Idealized Heterogeneous Partition}: all points $x_i \in \mcD_i$  and $x_j \in \mcD_j$ are separated ($k(x_i, x_j) \approx 0$). In other words, the local datasets form clusters in the input space.
\end{itemize}

\noindent Our primary result is that the BCM equations under-estimate the predictive variance in the case of the idealized \emph{homogeneous} partition, i.e., the BCM predictions are overconfident:

\begin{restatable}[\textbf{BCM, homogeneous}]{theorem}{BCMHomogeneous}
    If the data is split among $m$ clients in an idealized homogeneous partition, and $x^* \in R$ is a single test point, then the BCM equations \eqref{eq:var_regr} underestimate the predictive variance $\sigma^2_{\bcm}(x^*) < \sigma^2_o$ as the size of the data subsets grows ($|\mcD_i| \to \infty$). 
\end{restatable}

On the other hand, in the same setting, for the mean prediction, if we assume a small prior precision $\sigma^{-2}_p \approx 0$, a quick calculation shows that we still obtain accurate results. Starting with the equation for the BCM predictive mean \eqref{eq:var_regr} and applying Lemma \ref{lemma:obs_var}, we obtain:
\begin{align*}
 \small  \mu_{\bcm}(x^*) &= \sum_i \sigma^{-2}_i \mu_i / (\sum_i \sigma^{-2}_{i}) \\
               &\approx \sum_i \sigma^{-2}_o f(x^*) / (\sum_i \sigma^{-2}_o) = f(x^*).
\end{align*}
In other words, the predictive mean still recovers the true mean function at $x^*$. We can also check that the BCM produces a calibrated predictive variance in the case of idealized heterogeneous data, which lines up with the idea that the BCM is accurate under certain \emph{heterogeneous} partitions. Note that this contrasts with proposition 1 of \citet{liu18a}, which asserts that the BCM is overconfident with heterogeneous partitions. The reason for the difference is that they assume i) that the data stays in the same bounded region as $|\mcD| \to \infty$, and ii) that the number of partitions grows to $\infty$. In that setting, the disjoint partitions clump together as $|\mcD| \to \infty$, unlike our setting. 

\begin{restatable}[\textbf{BCM, heterogeneous}]{theorem}{BCMHeterogeneous}
If the data is split among $m$ clients in an idealized heterogeneous partition, and $x^* \in R$ is a test point near $\mcD_k$, then the BCM equations \eqref{eq:var_regr} correctly estimate the predictive variance $\sigma^2_{\bcm}(x^*) = \sigma^2_o$ as the size of the data subsets grows ($|\mcD_i| \to \infty$). 
\end{restatable}

\subsection{Analyzing the Predictive Mixture Model}

From the above we see that the BCM model produces overconfident estimates for the (idealized) homogeneous setting. An alternate model that produces well-calibrated estimates in this setting is the \textbf{predictive mixture model}:
\begin{equation}
    \label{eq:mixture_model}
    \sum_i \frac{|\mcD_i|}{|\mcD|} p(y|x,\mcD_i). 
\end{equation}
For GP regression, when each $p(y|x,\mcD_i) = \mathcal{N}(\mu_i(x), \sigma^2_i(x))$ is Gaussian, this model produces a Gaussian mixture as its predictive distribution. In general, this will not be Gaussian, but we can approximate it with an overall Gaussian prediction $\mathcal{N}(\mu_{\mix}(x), \sigma^2_{\mix}(x))$, matching the mean and variance of the mixture:
\begin{align}
    \label{eq:mixture_gaussian}
    \mu_{\mix}(x) &= \sum_i \frac{|\mcD_i|}{|\mcD|} \mu_i(x),\\
    \sigma^2_{\mix}(x) &=\sum_i \frac{|\mcD_i|}{|\mcD|} (\sigma^2_i(x) + \mu^2_i(x)) - \mu^2_{\mix}(x).
\end{align}
This is because the moments of a mixture are equal to the mixtures of the moments.  Unlike the BCM model, for homogeneous data, the mixture model is well-calibrated:

\begin{restatable}[\textbf{mixture model, homogeneous}]{theorem}{MixtureHomogeneous}
If the data is split among $m$ clients in an idealized homogeneous partition, and $x^* \in R$ is a test point in the bounded region of the data, then the predictive mixture equations \eqref{eq:mixture_gaussian} correctly estimate the predictive variance $\sigma^2_{\mix}(x^*) = \sigma^2_o$ as the size of the data subsets grows ($|\mcD_i| \to \infty$). 
\end{restatable}

\noindent However, the drawback of the mixture model is that for the idealized heterogeneous partition, it overestimates the predictive uncertainty:

\begin{restatable}[\textbf{mixture model, heterogeneous}]{theorem}{MixtureHeterogeneous}
If the data is split among $m$ clients in an idealized heterogeneous partition, and $x^* \in R$ is a test point near $\mcD_k$, then the predictive mixture equations \eqref{eq:mixture_gaussian} overestimate the predictive variance $\sigma^2_{\mix}(x^*) > \sigma^2_o$ as the data subset size grows ($|\mcD_i| \to \infty$). 
\end{restatable}

\subsection{Calibration Analysis for Classification}

The preceding analysis may be extended to classification (or more general) models under the following assumptions. Below, we denote the true underlying predictive model (with correct aleatoric uncertainty) as $p_{T}(y|x)$, and the prior predictive model as $p_P(y|x)$. We assume:

\begin{enumerate}
    \item For $x^*$ in the vicinity of data subset $\mcD_i$, and assuming all data subsets are large enough, the local predictive model converges to the true model, i.e. $p(y|x,\mcD_i) = p_{T}(y|x)$;

    \item For $x^*$ far away from the data subset $\mcD_i$, the local predictor outputs the prior model: $p(y|x, \mcD_i) = p_{P}(y|x)$;

    \item The prior model $p_P(y|x)$ has higher variance, uncertainty, and/or entropy than the true model $p_T(y|x)$.
\end{enumerate}

All these are desiderata for well-calibrated models. The first desideratum ensures that the model is accurate for in-domain points and the second desideratum ensures that the model reverts to predictions it would have made in the absence of data (the prior predictions) for out-of-domain points. The final desideratum ensures that the prior predictions are sufficiently uncertain. 

Although some of these assumptions may not be satisfied by some approximate models, they are the ideal goal for a well-calibrated model. For instance, in classification with neural networks using a softmax layer, item 2 may not be true, and it is the goal of multiple methods to correct this defect \cite{mukhoti2021deterministic}. This is why, for analytically simple models such as GP regression, these assumptions hold.

Under the above assumptions of well-calibrated local models, we can extend the analysis to classification. For simplicity we assume $p_{P}(y|x)$ is a uniform distribution. We let $c$ denote the correct class for the test point $x^*$.
\begin{restatable}[\textbf{BCM, homogeneous, classification}]{theorem}{BCMHomoClass}
Assuming well-calibrated local models, and an idealized homogeneous partition, if $x^* \in R$ is a test point in the bounded region of the data, then BCM predicts the correct class with probability $p_{\bcm}(c | x^*, \mcD)  > p_T(c | x^*, \mcD)$.
\end{restatable}
This follows from the fact the local distributions are correct $p(y|x,\mcD_i) = p_T(y|x)$ and BCM multiplies them together to produce $p_{\bcm}(y|x,\mcD) \propto p^m_T(y|x)$, which will be more confident than the desired distribution $p_T(y|x)$.

For heterogeneous data, only one of the factors in the product distribution converges to $p_T(y|x)$, while the others still match the uniform prior $p_{P}(y|x)$. In this case, the product yields an accurate model.
\begin{restatable}[\textbf{BCM, heterogeneous, classification}]{theorem}{BCMHeteroClass}
Assuming well-calibrated local models, and an idealized heterogeneous partition, if $x^* \in R$ is a single test point near $\mcD_k$, then the BCM predicts the correct class with probability $p_{\bcm}(c | x^*, \mcD)  = p_T(c | x^*, \mcD)$.
\end{restatable}

\noindent Similarly analyzing the mixture model, we see it is underconfident for heterogeneous data:
\begin{restatable}[\textbf{mixture model, heterogeneous, classification}]{theorem}{MixtureHeteroClass} 
Assuming well-calibrated local models, and an idealized heterogeneous partition, if $x^* \in R$ is a test point near $\mcD_k$, then the mixture predicts the correct class with probability $p_{\mix}(c | x^*, \mcD)  < p_T(c | x^*, \mcD)$.
\end{restatable}

\noindent On the other hand, for homogeneous data, each predictive distribution is identical, and equal to the true distribution, so that the mixture returns the correct distribution $p_{\mix}(y|x, \mcD) = \sum_i (|\mcD_i|/|\mcD|) p_T(y|x) = p_T(y|x)$.
\begin{restatable}[\textbf{mixture model, homogeneous, classification}]{theorem}{MixtureHomoClass}
Assuming well-calibrated local models, and an idealized homogeneous partition, if $x^* \in R$ is a test point in the bounded region of the data, then the mixture predicts the correct class with probability $p_{\mix}(c | x^*, \mcD)  = p_T(c | x^*, \mcD)$.
\end{restatable}
\noindent We summarize the analysis results in Table \ref{tab:analysis_summ}.
\begin{table}
\centering
\scalebox{0.9}{
\begin{tabular}{ccc}
\toprule
& \textbf{Heterogeneous Data} & \textbf{Homogeneous Data}\\
\midrule
\textbf{BCM} & Calibrated & Overconfident  \\
\midrule
\textbf{Mixture} & Underconfident & Calibrated \\
\bottomrule
\end{tabular}
}
\caption{Summary of Predictive Variance Analysis.}
\label{tab:analysis_summ}
\end{table}

\section{Calibrating the Aggregated Model}

\begin{small}
\begin{figure*}[!htb]
    \centering
    \includegraphics[width=0.8\linewidth]{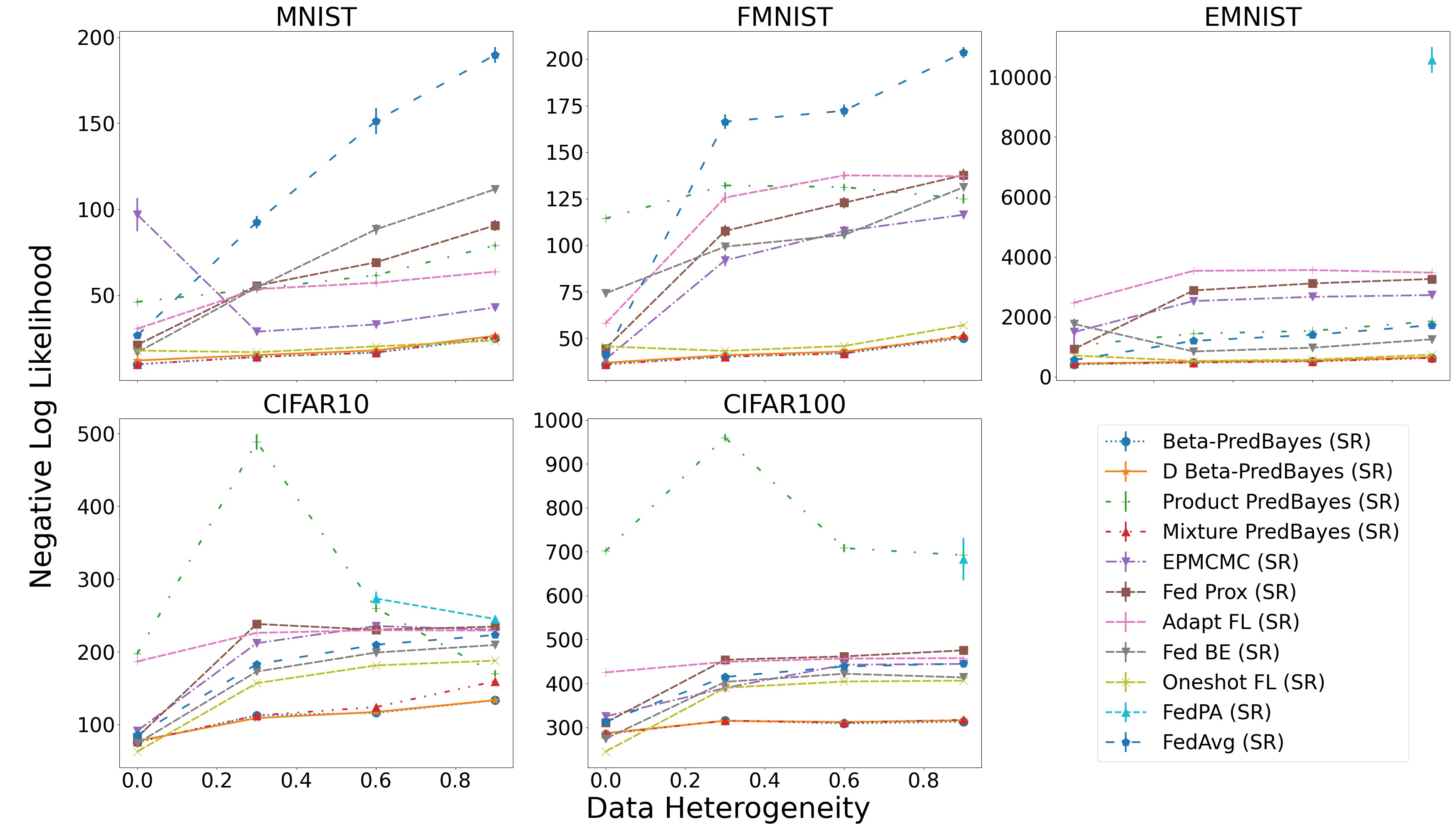}
    \caption{NLL on the classification datasets with increasing heterogeneity (tested with $h\in$\{0.0, 0.3, 0.6, 0.9\}). Averages and standard error over 10 seeds are reported. Omitted values (e.g., for FedPA on EMNIST) denote results where NLL diverged.}
    \label{fig:noniid_class_nllhd}
\end{figure*}
\end{small}
\begin{small}
\begin{figure*}[!htb]
    \centering
    \includegraphics[width=0.8\linewidth]{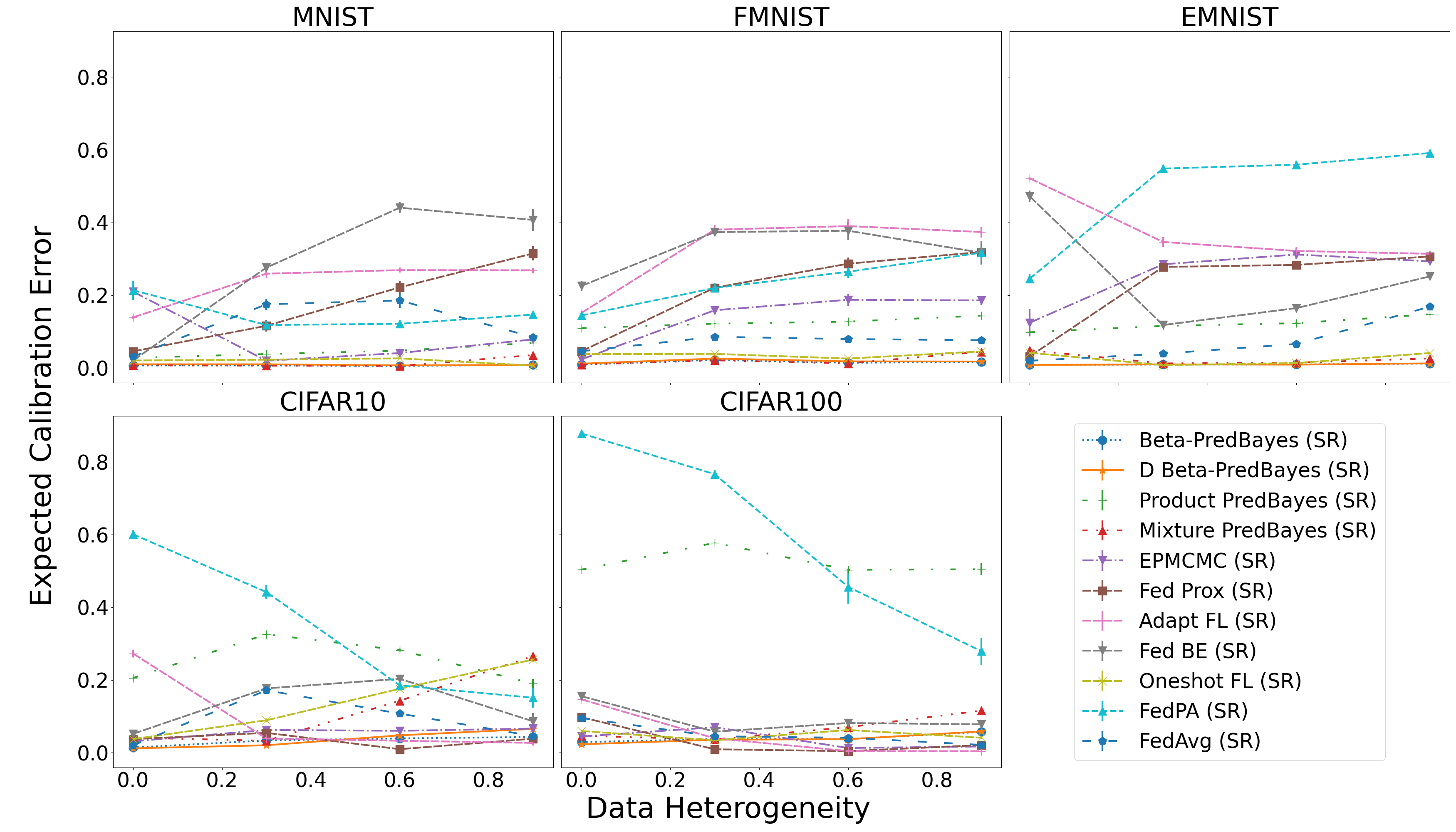}
    \caption{ECE on the classification datasets with increasing heterogeneity (tested with $h\in $\{0.0, 0.3, 0.6, 0.9\}). Averages and standard error over 10 seeds are reported.}
    \label{fig:noniid_class_cal}
\end{figure*}
\end{small}

The preceding analysis motivates us to combine the predictive mixture model \eqref{eq:mixture_model}, and the BCM (henceforth referred to as the ``product model") \eqref{eq:prod_rule}. Namely, to obtain the correct calibration in the predictive model, we should interpolate between the mixture, which is accurate for homogeneous data, and the product, which is accurate for heterogeneous data. For interpolation parameter $\beta$, the model, which we refer to as \emph{$\beta$-Predictive Bayes} is:
\begin{align}
    \label{eq:interp_model}
    \log p_{\beta}(y|x,\mcD) &= \beta \log \Big( \frac{1}{p(y|x)^{n-1}} \prod_i p(y | x,\calD_i) \Big) +\\ &(1-\beta) \log \Big( \sum_i \frac{|\mcD_i|}{|\mcD|} p(y|x,\mcD_i) \Big). \nonumber
\end{align}
The case of $\beta = 0.0$ corresponds to the mixture model, and $\beta = 1.0$ corresponds to the product model. 

In the case of regression, assuming Gaussian outputs for each local predictive distribution $p(y|x,\mcD_i)$, and using the mixture approximations \eqref{eq:mixture_gaussian}, the aggregated predictive distribution is approximately a Gaussian $p_{\beta}(y|x,\mcD) = \mathcal{N}(\mu_\beta(x), \sigma^2_{\beta}(x))$ with:
\begin{align}
    \label{eq:interp_gaussian}
    \sigma^{-2}_{\beta}(x) &= \beta\cdot \sigma^{-2}_{\prd} + (1-\beta) \cdot \sigma^{-2}_{\mix}(x)  \\
    \mu_{\beta}(x) &= \sigma^2_{\beta}(x) \big( \beta \cdot \sigma^{-2}_{\prd}(x) \mu_{\prd}(x) \nonumber\\&+ (1-\beta) \cdot \sigma^{-2}_{\mix}(x) \mu_{\mix}(x) \big). \nonumber 
\end{align}
In other words, the inverse-variance (or precision) interpolates between that of the product and mixture distributions.

We learn $\beta$ by minimizing the negative log-likelihood of a dataset on the server, $\mathcal{U}$ using a gradient-based optimizer.
\begin{align}
    \beta^* = \argmin_{\beta} \sum_{(x,y) \in \mathcal{U}} -\log p_{\beta}(y|x,\mcD). 
\end{align}
For regression, the Gaussian negative log-likelihood can be used for training (by approximating $p_{\beta}(y|x,\mcD)$ with \eqref{eq:interp_gaussian}).

By tuning $\beta$ this way, we don't need to know where we are along the homogeneous-heterogeneous partition spectrum. Since we are only tuning a single scalar parameter, a small dataset may suffice. Furthermore, we can use $\mathcal{U}$ as a distillation dataset to compress the ensemble into one model. The steps for $\beta$-predBayes are presented in Algorithm \ref{alg:dBpredbayes}. 

\begin{algorithm}[tb]
\caption{Distilled $\beta$-PredBayes}
\label{alg:dBpredbayes}
\textbf{Input}: Client datasets $\mcD_i$, sampler $\mathrm{MCMC\_sample}$\\
\textbf{Output}: Model $\theta^*$
\begin{algorithmic} 
\FOR {each client $i$}
\STATE $\{\theta\}_i = \mathrm{MCMC\_sample}(\mcD_i)$  \hfill\COMMENT{//step 1}
\STATE Communicate $\{\theta\}_i$ to server \hfill\COMMENT{//step 2}
\ENDFOR

\STATE At Server: 
\STATE $\hat{p}(y|x,\mcD_i) = \frac{1}{|\{\theta\}_i|}\sum_{\theta \in \{\theta\}_i} p(y|x, \theta) $ \hfill \COMMENT{//step 3}
\STATE $\hat{p}_{\beta}(y|x, \mcD) = \mathrm{Aggregate}( \hat{p}(y|x,\mcD_i ))$ \hfill \COMMENT{//step 4, Eq.~\ref{eq:interp_model}}
\STATE $\beta^* = \argmin_{\beta} \sum -\log \hat{p}_{\beta}(y|x,\mcD)$ \hfill \COMMENT{//step 5, tune $\beta$}
\STATE $\theta^* = \mathrm{Distill}(\hat{p}_{\beta^*}(y|x,\mcD))$ \hfill\COMMENT{//step 6} 
\STATE \textbf{return} $\theta^*$
\end{algorithmic}
\end{algorithm}

\section{Experiments}

\begin{small}
\begin{table*}[!htb] 
    \centering
    \label{tab:regr_nll}
    \setlength{\tabcolsep}{5pt}
   \scalebox{0.98}{
    \begin{tabular}{lclllll}
        \toprule
        Method & \multicolumn{1}{c}{Air Quality} & \multicolumn{1}{c}{Bike} & \multicolumn{1}{c}{Wine Quality} & \multicolumn{1}{c}{Real Estate} & \multicolumn{1}{c}{Forest Fire} \\
        \midrule
        EP-MCMC & 11.20 \small{$\pm$ 0.42}      & 1.92 \small{$\pm$ 0.12}  & 2.93 \small{$\pm$ 0.09} & 1.90 \small{$\pm$ 0.29}  &    1.90 \small{$\pm$ 0.09} \\
        FedBE & 6.34 \small{$\pm$ 0.03} & 0.94 \small{$\pm$ 0.05}  & 2.13 \small{$\pm$ 0.01} &  0.49 \small{$\pm$ 0.01}  &     1.39 \small{$\pm$ 0.01} \\
        OneshotFL & 6.69 \small{$\pm$ 0.05} & 0.98 \small{$\pm$ 0.05} & 2.17 \small{$\pm$ 0.01} & 0.53 \small{$\pm$ 0.01}  &    1.39\small{ $\pm$ 0.01} \\
        Mixture & 9.45 \small{$\pm$ 0.02}   & 1.22 \small{$\pm$ 0.02} & 2.57 \small{$\pm$ 0.03} & 0.65 \small{$\pm$ 0.01}  &    \textbf{1.39}\small{ $\pm$ 0.01} \\
        Product &  10.17 \small{$\pm$ 0.08}   & 1.47 \small{$\pm$ 0.02} & 3.06 \small{$\pm$ 0.10} & 0.83 \small{$\pm$ 0.03}  &    2.57 \small{$\pm$ 0.03} \\
        \midrule
        $\beta$-PredBayes (ours) & 7.05\small{ $\pm$ 0.04} &  0.95 \small{$\pm$ 0.01}  & 2.05 \small{$\pm$ 0.05}  &  0.53 \small{$\pm$ 0.03}  & 1.55\small{ $\pm$ 0.01}$^{*}$ \\
        D-$\beta$-PredBayes (ours) & \textbf{4.53} \small{$\pm$ 0.12} &  \textbf{0.32} \small{ $\pm$ 0.05}  & \textbf{1.34} \small{$\pm$ 0.03}    &    \textbf{0.18} \small{$\pm$ 0.02} & 1.56 \small{$\pm$ 0.03} \\
        \bottomrule
    \end{tabular}
    }
     \caption{Average NLL ($\pm$ standard error) on regression datasets, based on 10 seeds. Lower is better. All values are statistically significant relative to D-$\beta$-PredBayes ($p<5\%$) according to the Wilcoxon signed-rank test, except those marked with $*$.} 
\end{table*}
\end{small}

We evaluate $\beta$-PredBayes on a number of regression and classification datasets. All tests used 5 clients, with a distillation set composed of $20\%$ of the original training set. Further experimental details, such as the models, and hyperparameters, are included in the appendix.

Since the goal is to evaluate the calibration of the predictions (not just their accuracy), we measure negative-log-likelihood (NLL) on the test set. This evaluates calibration since NLL should be minimized for the correct prediction probabilities \cite{guo2017calibration}. For classification, we may also use the expected calibration error (ECE), which measures the difference between the confidence (probabilities) assigned to predicted classes, and the fraction that are predicted correctly (see the appendix for more details).

\subsection{Classification Dataset Setup} The method was evaluated for classification on the following datasets: MNIST~\citep{726791}, Fashion MNIST~\citep{xiao2017/online}, EMNIST~\citep{cohen2017emnist} (using a split with 62 classes), CIFAR10 and CIFAR100~\citep{krizhevsky2009learning}. Each of these datasets is distributed among clients based on a ``heterogeneity parameter", $h$. For $h=0$, the data is split uniformly among the clients, and is thus homogeneously distributed. For $h=1$, the data is sorted by class before being split among clients. This means that each client observes data from different classes with little overlap. In this sense, $h=1$ represents the ``fully heterogeneous" setting. For $0 < h < 1$, data from the above extremes is mixed: a fraction of size $h$ of the homogeneous data shard is replaced with the fully heterogeneous data for each client.

\subsection{Classification Results}

We evaluate $\beta$-PredBayes, in both its distilled (listed as D-$\beta$-PredBayes) and non-distilled forms, along with the product (BCM) and mixture models. We also evaluate several baselines, run for one communication round. 
The results for NLL over different settings of the heterogeneity are plotted in Figure \ref{fig:noniid_class_nllhd}, while ECE results are shown in Figure \ref{fig:noniid_class_cal}. 

For the NLL, the $\beta$-PredBayes model (and its distilled variant) perform best (least NLL), followed by the Mixture model, on the tested datasets. As heterogeneity increases, the NLL loss generally increases for other methods, while it largely remains stable for $\beta$-PredBayes. For the ECE, as heterogeneity increases, the metric jumps more erratically for some methods (which may be due ECE's sensitivity to hyperparameters like bin count). But the overall trend is still that $\beta$-PredBayes performs best (with lowest ECE), followed by its distilled variant, followed by the mixture model. We note that for ECE, $\beta$-PredBayes outperforms the mixture model for high heterogeneity setting ($h=0.9$), which is expected from the analysis that predicted that mixture models would not be well-calibrated in this setting.

\subsection{Regression Dataset Setup} The regression datasets used for evaluation include: the ``wine quality"~\citep{cortez2009modeling}, ``air quality"~\citep{de2008field}, ``forest fire"~\citep{cortez2007data}, ``real estate"~\citep{yeh2018building}, and ``bike rental"~\citep{fanaee2013event} datasets from the UCI repository~\citep{uci}. These datasets were sorted according to certain features (such as the date, for ``airquality"), then split among clients, to simulate heterogeneous data.

\subsection{Regression Results}

For regression, $\beta$-predBayes was evaluated, along with the product and mixture models as well as other baselines that use an ensemble for predictive inference (since these yield a predictive variance). The resulting (Gaussian) NLL for heterogeneous data are shown in Table 2. For all datasets except ``Forest Fire," the distilled form of $\beta$-PredBayes performs best in NLL. Note that it is possible that the distilled version outperforms the non-distilled one due to some regularization effect of having a smaller model.

\section{Conclusion and Future Work}

This work presented $\beta$-Predictive Bayes, an algorithm that aggregates local Bayesian predictive posteriors using a tunable parameter $\beta$, and then distills the resulting model. The parameter $\beta$ enables accurate calibration performance. Owing to its Bayesian nature, the method is also well suited for heterogeneous FL, and requires only a single communication round. We perform empirical evaluation on various classification and regression datasets to show the competence of our methods in heterogeneous settings. Our work reinforces that Bayesian learning can provide well-calibrated models in heterogeneous settings with efficient communication.

Some future directions to improve our work include:

\begin{itemize}
    \item \textbf{Personalized FL.} Our proposed approach learns a single global model for all clients. When clients have different class-conditional distributions $p(y|x)$, a personalized model can be more desirable for each client. 

    \item \textbf{Privacy.} The proposed method transmits weight samples (obtained by MCMC) to the server. Although no data is shared, information about the data could be leaked via the model weights unless a theoretically proven private mechanism is used or encryption is applied. This trait is shared by other techniques (e.g., FedAvg, FedPA). It would be interesting to explore rigorously private methods, such as differentially private sampling mechanisms~\cite{dimitrakakis2017differential}. 

    \item \textbf{Server dataset.} $\beta$-Predictive Bayes uses public data stored at the server for distillation and tuning $\beta$. When a public dataset is not available, it may need to be synthetically generated. It would be better to avoid this generation and to develop a data-free technique for distillation and learning $\beta$ in the future. 
\end{itemize}

\section*{Acknowledgments}

Resources used in preparing this research at the University of Waterloo were provided by Huawei Canada, the Natural Sciences and Engineering Research Council of Canada, the Province of Ontario, the Government of Canada through CIFAR, and companies sponsoring the Vector Institute.

\appendix

\bibliography{aaai24}

\setcounter{secnumdepth}{1}
\setcounter{lemma}{0}
\setcounter{theorem}{0}

\section{Proofs of Theorems}

We provide proofs for the theorems in the main text. 
Recall the BCM aggregation equation as:

\begin{align*}
    p(y|x,\mcD) &= \frac{1}{p^{m-1}(y|x)} \prod_{i} p(y|x, \mcD_i) 
\end{align*}
which, in the case of Gaussian local predictive posteriors $p(y|x,\mcD_i) = \mathcal{N}(\mu_i, \Sigma_i)$ simplifies to computing the Gaussian global predictions $p(y|x,\mcD) = \mathcal{N}(\mu_g, \Sigma_g)$:
\begin{align*}
    \Sigma_g &= \left(\sum_i \Sigma^{-1}_i - (n-1)\Sigma^{-1}_p \right)^{-1} \\
    \mu_g &= \Sigma_g \left(\sum_i \Sigma^{-1}_i \mu_i - (n-1)\Sigma^{-1}_p\mu_p\right). 
\end{align*}
\noindent
The mixture aggregation is:
\begin{equation*}
    \sum_i \frac{|\mcD_i|}{|\mcD|} p(y|x,\mcD_i)
\end{equation*}
which we approximate with an overall Gaussian prediction $\mathcal{N}(\mu_{\mix}(x), \sigma^2_{\mix}(x))$:
\begin{align*}
    \mu_{\mix}(x) &= \sum_i \frac{|\mcD_i|}{|\mcD|} \mu_i(x) \\
    \sigma^2_{\mix}(x) &=\sum_i \frac{|\mcD_i|}{|\mcD|} (\sigma^2_i(x) + \mu^2_i(x)) - \mu^2_{\mix}(x).  
\end{align*}

\subsection{Gaussian Process Analysis}

We assume a smooth, isotropic kernel function $k(x,x') = k(||x - x'||)$. We assume a model with Gaussian noise, i.e., $y = f(x) + \epsilon$  where $\epsilon \sim \mathcal{N}(0, \sigma^2_o)$.

On some dataset $\mcD = (\mathbf{X}, \mathbf{y})$, and with test point $x^*$, the Gaussian process inference equations predict a mean and variance for $p(y^*|x^*, \mcD_i) = \mathcal{N}(\mu(x^*), \sigma^2(x^*))$. If we denote $\mathbf{k}_* = [k(x^*, x_1), ..., k(x^*, x_{|\mcD|})]^\top$ the vector of kernel evaluations between the test point $x^*$ and the points in the dataset $\mcD$, and $\mathbf{K}$ as the kernel matrix (where $K_{i,j} = k(x_i, x_j)$ for $x_i, x_j \in \mcD$) then the inference equations are \cite{rasmussenGP}:
\begin{align}
    \label{eq:gp_inference}
    \mu(x^*) &= \mathbf{k}^\top_* (\mathbf{K} + \sigma^2_o I)^{-1} \mathbf{y}\\
    \sigma^2(x^*) &= \sigma^2_o + k(x^*, x^*) -  \mathbf{k}^\top_* (\mathbf{K} + \sigma^2_o I)^{-1}\mathbf{k}_*
\end{align}
\noindent We assume the data lies in some bounded set $R$. Starting with the lemmas in the main text:

\LemmaOne*

\LemmaTwo*

\begin{proof}
    The inference equation \ref{eq:gp_inference} for the predictive variance reads:
    
    \begin{align*}
        \sigma^2(x^*) &\approx \sigma^2_o + k(x^*, x^*) - \mathbf{0}^\top(\mathbf{K} + \sigma^2_oI)^{-1}\mathbf{0} \\
        &\approx \sigma^2_o + k(x^*, x^*) 
    \end{align*}
    
    Similarly, the predictive mean is $\mu(x^*) \approx \mathbf{0}^\top (\mathbf{K} + \sigma^2_o I)^{-1} \mathbf{y} = 0$.
    \end{proof}

\subsubsection{BCM}

The primary theorems establishing the calibration performance of the BCM and mixture models are:

\BCMHomogeneous*

\begin{proof}
    According to the BCM aggregation equation for regression:

    \begin{align*}
         \sigma^{-2}_{\bcm}(x^*) &= \sum_i \sigma^{-2}_i(x^*) - (m-1)\sigma^{-2}_p(x^*) 
    \end{align*}
    Where $\sigma^{-2}_i(x^*)$ is the inverse of the predictive variance output by the GP trained on $\mcD_i$. As the size of the dataset $\mcD_i$ increases, from Lemma \ref{lemma:obs_var} we know that $\sigma_i(x^*)$ will converge to $\sigma^2_o$. On the other hand $\sigma^2_p = k(x^*, x^*) + \sigma^2_o$. Combining these facts (in the limit of increasing data points):

    \begin{align*}
         \sigma^{-2}_{\bcm}(x^*) &= m \sigma^{-2}_o - (m-1)(\sigma^{2}_o + k(x^*, x^*))^{-1} \\
                                 &> m \sigma^{-2}_o - (m-1)\sigma^{-2}_o \\
                                 &= \sigma^{-2}_o
    \end{align*}
    Where the second inequality follows from the fact that $k(x^*, x^*) > 0$ (for a positive kernel function).
    
    This implies $\sigma^2_{\bcm} < \sigma^2_o$ as claimed. 
    
    Furthermore, for a kernel function with a large prior precision $\sigma^{-2}_p \approx 0$ (which is typical in practice when using an uninformative prior), we see that: 
    \begin{align*}
        \sigma^{-2}_{\bcm} &= m \sigma^{-2}_o - (m-1) \sigma^{-2}_p \\
                           &\approx m \sigma^{-2}_o.
    \end{align*}

So we end up underestimating the predictive variance by a factor of $m$, for $m$ clients.
\end{proof}

\BCMHeterogeneous*

\begin{proof}
    The BCM aggregation equation for regression is:

    \begin{align*}
         \sigma^{-2}_{\bcm}(x^*) &= \sum_i \sigma^{-2}_i(x^*) - (m-1)\sigma^{-2}_p(x^*) 
    \end{align*}

    Since $x^*$ is in the vicinity of some data subset $\mcD_k$, we can apply Lemma \ref{lemma:obs_var} to get $\sigma^{-2}_k(x^*) = \sigma^{-2}_o$. On the other hand, since the data subsets are split up heterogeneously, the test point $x^*$ is far from the other data sub-sets such that we can apply Lemma \ref{lemma:prior_var} to obtain $\sigma^{-2}_i(x^*) = \sigma^{-2}_p(x^*)$ for all $i\neq k$. Plugging these into the above expression yields:

    \begin{align*}
         \sigma^{-2}_{\bcm}(x^*) &=  \sigma^{-2}_k(x^*) + \sum_{i\neq k} \sigma^{-2}_i (x^*)  - (m-1)\sigma^{-2}_p(x^*) \\
         &= \sigma^{-2}_o + (m-1) \sigma^{-2}_p(x^*) - (m-1)\sigma^{-2}_p(x^*) \\
         &= \sigma^{-2}_o.
    \end{align*}

\end{proof}

\subsubsection{Mixture Model}

The following theorems analyze the calibration performance of the mixture model.

\MixtureHomogeneous*
\begin{proof}

In this setting, appealing to Lemma \ref{lemma:obs_var}, each local predictor outputs $\mathcal{N}(\mu_i(x^*) = f(x^*), \sigma^2_i(x^*) = \sigma^2_o)$. Therefore we have:
\begin{align*}
    \sigma^{2}_{\mix} &=  \sum_i \frac{|\mcD_i|}{|\mcD|} (\sigma^2_i(x^*) + \mu^2_i(x^*)) - \mu^2_{\mix}(x^*) \\
                      &=  \sum_i \frac{|\mcD_i|}{|\mcD|} (\sigma^2_o + f^2(x^*)) - (\sum_i \frac{|\mcD_i|}{|\mcD|}f(x^*))^2 \\
                      &= \sigma^2_o + f^2(x^*) - f^2(x^*) \\
                      &= \sigma^2_o.                      
\end{align*}
\end{proof}
\MixtureHeterogeneous*

\begin{proof}
    The mixture models predictive variance is:

    \begin{align*}
    \sigma^2_{\mix}(x^*) &= \sum_i \frac{|\mcD_i|}{|\mcD|} (\sigma^2_i(x^*) + \mu^2_i(x^*)) - \mu^2_{\mix}(x^*)
    \end{align*}

    Since $x^*$ lies in the vicinity of $\mcD_k$, we have (applying Lemma \ref{lemma:obs_var}) that $\sigma^2_k(x^*) = \sigma^2_o$ (and $\mu_k(x^*) = f(x^*)$), while from Lemma \ref{lemma:prior_var} we know $\sigma^2_i(x^*) = \sigma^2_p(x^*) > \sigma^2_o$ (and $\mu_i(x^*) = 0$) for all $i \neq k$.
    
    \begin{align*}
        \sigma^2_{\mix}(x^*) &= \frac{|\mcD_k|}{|\mcD|} (\sigma^2_o + f^2(x^*)) \\&~~+ \sum_{i\neq k} \frac{|\mcD_i|}{|\mcD|} (\sigma^2_p(x^*)) - \mu^2_{\mix}(x^*) \\
                            &= \frac{|\mcD_k|}{|\mcD|} \sigma^2_o  + \sigma^2_p(x^*) \Big(1 - \frac{|\mcD_k|}{|\mcD|}\Big) + \frac{|\mcD_k|}{|\mcD|} f^2(x^*) \\&~~- \frac{|\mcD_k|^2}{|\mcD|^2}f^2(x^*) \\
                            &= \frac{|\mcD_k|}{|\mcD|} \sigma^2_o  + \sigma^2_p(x^*) \Big(1 - \frac{|\mcD_k|}{|\mcD|}\Big) \\&~~ + (\frac{|\mcD_k|}{|\mcD|} - \frac{|\mcD_k|^2}{|\mcD|^2})f^2(x^*)\\
                            &> \frac{|\mcD_k|}{|\mcD|} \sigma^2_o  + (\sigma^2_o + k(x^*, x^*)) \Big(1 - \frac{|\mcD_k|}{|\mcD|}\Big) \\
                            &= \sigma^2_o + k(x^*, x^*) \Big(1 - \frac{|\mcD_k|}{|\mcD|}\Big) \\
                            &> \sigma^2_o.
    \end{align*}
    In other words, we overestimate $\sigma^2_o$ by a constant of at least $k(x^*, x^*) \big(1 - \frac{|\mcD_k|}{|\mcD|}\big) $.
\end{proof}

We note that in this case, the predictive mean of the model is also inaccurate since $\mu_{\mix}(x^*) = \frac{|\mcD_k|}{|\mcD|} f(x^*) \neq f(x^*)$. 

\subsection{Classification Analysis}

For classification, the following theorems determine the calibration performance of the BCM and mixture models (assuming a uniform prior). We make use of the fact that for the predicted class $p_T(c|x^*) > p_T(i|x^*)$, for $i \neq c$ (ie. it has the highest probability).

\BCMHomoClass*

    \begin{proof}
        Let $p_i =  p_T(y=i|x^*, \mcD)$ denote the true class probabilities. Since the local models are well-calibrated, each local model predicts $p(c|x^*, \mcD_i) = p_c$. The BCM prediction for the correct class is $p_{\bcm}(c|x^*,\mcD) \propto p^m_c$. We have:

        \begin{align*}
            p_{\bcm}(c|x^*, \mcD) &= p^m_c / (\sum_i p^m_i) \\
                                  &= 1 / (\sum_i (\frac{p_i}{p_c})^m) \\
                                  &> 1/(\sum_i \frac{p_i}{p_c}) \\
                                  &= p_c\\
                                  &= p_T(c|x^*, \mcD).
        \end{align*}  
        The inequality in line 3 above follows from the fact that $p_i < p_c$ for the correct class $c$, so that $p_i/p_c < 1$ and $(p_i/p_c)^m < p_i/p_c$.
    \end{proof}

    \BCMHeteroClass*
    
    \begin{proof}
        We have that for $x^*$ near $\mcD_k$, $p(c|x^*,\mcD_k) = p_T(c|x^*,\mcD)$ while for other shards $i\neq k$ $p(c|x^*,\mcD_i) = p_P(c|x^*) = 1/K$:
        \begin{align*}
            p_{\bcm}(c|x^*, \mcD) &= \frac{1}{p^{m-1}(c|x^*)}\prod_i p(c|x^*,\mcD_i) \\
            &= K^{m-1} p_T(c|x^*,\mcD) \prod_{i\neq k} \frac1K \\
            &= \frac{K^{m-1}}{K^{m-1}} p_T(c|x^*,\mcD) \\
            &= p_T(c|x^*,\mcD).
        \end{align*}
    \end{proof}

\MixtureHeteroClass*

    \begin{proof}
        We have that for $x^*$ near $\mcD_k$, $p(c|x^*,\mcD_k) = p_T(c|x^*,\mcD)$ while for other shards $i\neq k$ $p(c|x^*,\mcD_i) = p_P(c|x^*) = 1/K$. Letting $p_T(c|x^*,\mcD) = p_c$ and noting that $p_c > \frac{1}{K}$ (since otherwise the data point $x^*$ has no correct label):
    
        \begin{align*}
            p_{\mix}(c|x^*,\mcD) &= \sum_i \frac{|\mcD_i|}{|\mcD|} p(c|x^*,\mcD_i) \\
            &= \frac{|\mcD_k|}{|\mcD|} p_T(c|x^*,\mcD) + (1 - \frac{|\mcD_k|}{|\mcD|}) p_P(c|x^*) \\
            &=  \frac{|\mcD_k|}{|\mcD|} p_c + (1 - \frac{|\mcD_k|}{|\mcD|}) \frac{1}{K} \\
            &<   \frac{|\mcD_k|}{|\mcD|} p_c + (1 - \frac{|\mcD_k|}{|\mcD|}) p_c \\
            &= p_c.
        \end{align*}
    \end{proof}

    \MixtureHomoClass*
    \begin{proof}
        For each local dataset, $p(c|x^*,\mcD_i) = p_T(c|x^*,\mcD)$, so that:

        \begin{align*}
             p_{\mix}(c | x^*, \mcD)  &= \sum_i \frac{|\mcD_i|}{|\mcD|} p(c|x^*,\mcD_i) \\
             &= \sum_i \frac{|\mcD_i|}{|\mcD|} p_T(c|x^*,\mcD) \\
                                      &= p_T(c | x^*, \mcD).
        \end{align*}
    \end{proof}

\section{Additional Experimental Details}

Additional details for the experiments are provided below. 

\subsection{Hardware, Software, and Randomization Details}

The code for experiments was written in the Python language (version 3.8.10), 
primarily using the Pytorch (version 1.9.0), Numpy (version 1.19.5) and Scipy (version 1.6.2) libraries. Randomization was done by setting seeds for Pytorch and Numpy.

Experiments were carried out on a compute cluster using a single Nvidia GPU (either the T4, or P100).

\subsection{Expected Calibration Error}

Assuming the predicted class probabilities are organized into $M$ equally spaced bins $B_1, ..., B_M$ on the interval $[0,1]$, the expected calibration error is calculated as:
\begin{align}
    \sum^M_{m=1} \frac{|B_m|}{|\mcD|} \large| \mathrm{acc}(B_m) - \mathrm{conf}(B_m) \large|,
\end{align}
where $\mathrm{acc}(B_m)$ is the average accuracy of the predictions in bin $m$ (the fraction of points in $B_m$ classified correctly) and $\mathrm{conf}(B_m)$ is the average confidence of the predictions in bin $m$ (the average value of $p(c|x^*,\mcD)$, where $c$ denotes the predicted class).
For our experiments, we calculated ECE using 15 equally spaced bins.

\subsection{Models} A two-layer fully connected network with 100 hidden units was used for MNIST, Fashion MNIST, EMNIST, as well as the regression datasets. For the CIFAR10 and CIFAR100 datasets a Convolutional Neural network was used with 3 convolution layers, each followed by ``Max Pooling" layers, with a single fully connected layer at the end. For all networks, the ReLU activation function was used between layers. For classification, the output was the predictive distribution over classes, while for regression the output was the mean of the predictive distribution. A Gaussian prior is assumed over the network parameters, $p(\theta) = \mathcal{N}(0, \sigma^2 I)$, with $\sigma = \texttt{5e4}$.

\subsection{Baselines} The Federated techniques compared include: Federated Averaging (FedAvg, \citealt{mcmahan17a}, Federated Posterior Averaging (FedPA, \citealt{al-shedivat2021federated}, Embarrassingly Parallel MCMC (EP MCMC, \citealt{ep_mcmc}, FedProx \citep{li2020fedprox}, Adaptive FL \citep{reddi2021adaptive}, Federated Bayesian Ensemble (FedBE) \citep{chen2021fedbe}, One Shot Federated Learning (OneshotFL, \citealt{guha2019}, in addition to the BCM \cite{tresp2000bcm} and the mixture model. 

In the case of FedAvg, FedProx, FedBE, and OneshotFL, either SGD with momentum or the Adam optimizer~\citep{adam} was used for local optimization (as selected by a grid search). For the rest of the methods, including our own, since they require MCMC sampling, cyclic stochastic gradient Hamiltonian Monte-Carlo (cSGHMC, \citealt{zhang2020csgmcmc}) was used. For EP MCMC, the algorithm was computationally intractable for neural network models due to the calculation of the inverse of a covariance matrix over parameters. Thus a diagonal covariance matrix was assumed (which corresponds to the assumption that the local posteriors are approximated by an axis-aligned Gaussian). All methods were run for a single round.

\subsection{Training Details} 

For MNIST, Fashion MNIST and EMNIST, the training was run for 25 epochs per client overall (split into 5 rounds for FedAvg and FedPA, while only run in a single round for the rest). For CIFAR10 and CIFAR100, training was run for 50 epochs per client (split into 10 rounds for multi-round methods). For the regression datasets, a total of 20 epochs are used for all datasets except ``air quality", which used 100 epochs (in both cases, divided into 5 rounds for multi-round methods). The methods involving sampling used a maximum of 6 samples for all experiments.

\subsection{Hyperparameter Tuning}

Hyperparameters were selected based on searching a grid for the best performing settings according to the validation set performance (accuracy for classification, and mean squared error for regression). 

The hyperparameters tuned, and their corresponding grids are outlined in Table \ref{tab:hparam_grid} for both classification and regression. Note that FedPA requires a sampler at each client, and an optimizer at the server. To distinguish where each hyperparameter is used for this algorithm, we therefore label these cases FedPA(C) and FedPA(S) respectively. 

In the tables, (D-)$\beta$-PredBayes denotes the (distilled) $\beta$-predictive Bayes algorithm.

The optimizers used for client training include SGD, SGD with momentum (SGDM), and Adam, while for distillation, we also used Stochastic Weight Averaging (SWA) as a possible optimizer.

The tuned hyperparameter settings for the homogeneous ($h=0$) classification datasets are in Table \ref{tab:tune_class}, while for the heterogeneous setting $h>0$ they  are in Table \ref{tab:tune_class_het}. For the regression datasets, the tuned values are in Table \ref{tab:tune_regr}.

\textbf{Note about reading tables}: for these tables, if a hyperparameter is repeated more than once, with an algorithm named beside it in brackets, it means the hyperparameter for that algorithm is different. The rest of the algorithms associated with that hyperparameter use the value listed without brackets. For instance, in Table \ref{tab:tune_regr}, for the ``Bike" dataset, the sampler learning rate is listed in the row ``Sampler Learning Rate" as \texttt{2e-1}, while it is listed separately with the additional specification ``FedPA (C)", as \texttt{5e-1}. This means that FedPA uses a sampler learning rate of \texttt{5e-1}, while the other sampling algorithms use \texttt{2e-1}.  \\

Other hyperparameters not part of the grid search include:

\begin{itemize}
    \item Batch size: fixed to $100$ for all experiments
    \item Momentum in SGDM: fixed to $0.9$ for all experiments
    \item Model architecture (as described in the main paper)
\end{itemize}

More algorithm-specific decisions/hyperparameters include:
\begin{itemize}
    \item FedBE: $10$ model samples were drawn from the approximate posterior to use in the ensemble for all experiments (following the experiments in the original paper. This gave a total ensemble size of $16$ models $= 10$ (sampled) + $5$ (client models) + $1$ (averaged model). By contrast $\beta$-predBayes contained an ensemble with $5$ models.
    \item Adaptive FL: The FedYogi server update was used, based on the results from, which suggested that it performed best among their proposed variants. $\beta_1 = 0.9$ and $\beta_2 = 0.99$ were fixed, again, based on the paper.
    \item One-Shot FL: For the classification case, aggregation is done by averaging the logits of the client models. (This is opposed to averaging and normalizing the probabilities, after the softmax layer). 
\end{itemize}

In tables \ref{tab:hparam_grid}, \ref{tab:tune_class}, \ref{tab:tune_class_het} and \ref{tab:tune_regr}, LR is used to denote learning rate.

\begin{table*}[hb]
    \centering
    \scalebox{0.9}{
    \begin{tabular}{lccc}
        \toprule
        \multirow{2}{*}{Hyperparameter}  & \multicolumn{2}{c}{Grid Settings} & \multirow{2}{*}{Algorithms Used In} \\
        \cmidrule(lr){2-3}
        & Classification & Regression & \\ 
        \midrule
        Optimizer & \multicolumn{2}{c}{\{SGD, SGDM, Adam\}} & \makecell{FedAvg, OneshotFL \\ FedPA(S), FedProx \\ AdaptFL, FedBE}  \\ \midrule 
        Local LR & \{\texttt{1e-1, 1e-2, 1e-3}\} &  \{\texttt{1e-1, 1e-2, 1e-3, 1e-4}\} & \makecell{FedAvg, OneshotFL,\\ FedProx, AdaptFL, FedBE} \\ \midrule
        Server LR & \multicolumn{2}{c}{\{\texttt{1, 5e-1, 1e-1, 1e-2}\}} & FedPA(S), AdaptFL \\ \midrule
        Cov. Param ($\rho$) & \multicolumn{2}{c}{\{\texttt{0.4, 0.9, 1.0}\}} & FedPA(C) \\ \midrule
        Proximal Parameter ($\lambda$) & \multicolumn{2}{c}{\{\texttt{1, 1e-1, 1e-2, 1e-3}\}} & FedProx \\ \midrule 
        Adaptivity ($\tau$) & \multicolumn{2}{c}{\{\texttt{1, 1e-1, 1e-2, 1e-3}\}} & AdaptFL \\ \midrule 
        Sampler LR & \small{\{\texttt{5e-1, 1e-1, 1e-2,1e-3}\}} & \small{\{\texttt{5e-1, 2e-1, 1e-1, 1e-2, 1e-3}\}} & \makecell{(D)PB, EPMCMC,\\ FedPA(C)} \\ \midrule 
        Maximum Samples & \multicolumn{2}{c}{\{\texttt{4,6,12}\}} & \makecell{(D)PB, EPMCMC,\\ FedPA(C)} \\ \midrule 
        Temperature & \{$\frac{1}{|\mcD|}$\} & \{\texttt{1, 5e-1, 5e-2}, $\frac{1}{|\mcD|}$ \} & \makecell{(D)PB, EPMCMC,\\ FedPA(C)} \\ \midrule
        Sampler Cycles & \{\texttt{5}\} & \{\texttt{2, 4,5}\} & \makecell{(D)PB, EPMCMC, \\ FedPA(C)} \\ \midrule
        Samples per Cycle & \{\texttt{2}\} & \{\texttt{1,2,3}\} & \makecell{(D)PB, EPMCMC, \\ FedPA(C)} \\ \midrule
        Distill Optimizer & \multicolumn{2}{c}{\{SGDM, Adam, SWA\}} & \makecell{D-PB, OneshotFL,\\ FedBE}  \\ \midrule
        Distill LR & \{\texttt{1e-2, 5e-3, 1e-4}\} &  \{\texttt{1e-2, 5e-3, 1e-3, 1e-4}\} & \makecell{D-PB, OneshotFL,\\ FedBE} \\ \midrule 
        Distill Epochs & \multicolumn{2}{c}{\{\texttt{100, 50, 20}\}} & \makecell{D-PB, OneshotFL,\\ FedBE} \\
        \bottomrule
    \end{tabular}
    }
    \caption{The hyperparameters tuned, their possible values in the grid search, and the algorithms each hyperparameter applies to. } 
    \label{tab:hparam_grid}
\end{table*}

\newcommand{\mc}[1]{\multicolumn{2}{c}{#1}}
\begin{table*}[ht]
    \centering
   \scalebox{0.8}{
    \begin{tabular}{lccccc}
        \toprule
        \multirow{2}{*}{Hyperparameter}  & \multicolumn{5}{c}{Tuned Value}  \\
        \cmidrule(lr){2-6} 
        & MNIST & Fashion MNIST & EMNIST & CIFAR10 & CIFAR100  \\
        \midrule
        Optimizer & \multicolumn{5}{c}{SGDM} \\ \midrule
        Optimizer (FedProx) & SGDM & Adam & Adam & SGDM & SGDM \\ \midrule
        Optimizer (AdaptFL) & SGDM & SGDM & SGDM & SGD & SGD \\ \midrule
        Local LR &\tt 1e-1 &\tt  1e-1 &\tt  1e-1 & \tt 1e-2 &\tt  1e-2  \\ \midrule
        Local LR (FedProx) & \tt 1e-1 & \tt 1e-3 & \tt 1e-3 &\tt  1e-2 & \tt 1e-2  \\ \midrule
        Local LR (AdaptFL) &\tt  1e-2 &\tt  1e-2 &\tt  1e-2 &\tt  1e-1 & \tt 1e-1  \\ \midrule
        Server LR &\tt  1 & \tt 5e-1 &\tt  1e-1 & \tt 5e-1 &\tt  5e-1  \\ \midrule
        Server LR (AdaptFL) & \multicolumn{5}{c}{\tt 1e-1}  \\ \midrule
        Cov. Param ($\rho$) & \multicolumn{5}{c}{\tt 0.4} \\\midrule 
        Proximal Param ($\lambda$) & \tt 1e-2 & \tt 1e-3 & \tt 1e-3 &\tt  1e-3 & \tt 1e-3 \\\midrule
        Adaptivity ($\tau$) & \multicolumn{5}{c}{1e-2} \\\midrule 
        Sampler LR &\tt  5e-1 &\tt  1e-1 &\tt  1e-1 & \tt 1e-1 &\tt  1e-1 \\ \midrule 
        Sampler LR (FedPA(C), EP MCMC) & \multicolumn{5}{c}{\tt 1e-1} \\ \midrule 
        Maximum Samples & \multicolumn{5}{c}{\tt 6}  \\ \midrule 
        Temperature & \multicolumn{5}{c}{$\frac{1}{|\mcD|}$} \\ \midrule
        Sampler Cycles & \multicolumn{5}{c}{\tt 5} \\ \midrule
        Samples per cycle & \multicolumn{5}{c}{\tt 2}  \\ \midrule
        Distill Optimizer & \multicolumn{5}{c}{Adam}  \\ \midrule
        Distill LR & \multicolumn{5}{c}{\tt 1e-4} \\ \midrule 
        Distill Epochs & \multicolumn{5}{c}{\tt 100} \\
        \bottomrule
    \end{tabular}
    }
    \caption{The tuned values of hyperparameters for the classification datasets in the homogeneous case $h=0$.} 
    \label{tab:tune_class}
\end{table*}

\begin{table*}[ht]
    \centering
     \scalebox{0.8}{
    \begin{tabular}{lccccc}
        \toprule
        \multirow{2}{*}{Hyperparameter}  & \multicolumn{5}{c}{Tuned Value}  \\
        \cmidrule(lr){2-6}
        & MNIST & Fashion MNIST & EMNIST & CIFAR10 & CIFAR100  \\ 
        \midrule
        Optimizer & \multicolumn{5}{c}{SGDM} \\ \midrule
        Optimizer (FedProx) & Adam & Adam & Adam & SGDM & SGDM \\ \midrule
        Optimizer (AdaptFL) & \multicolumn{5}{c}{SGD} \\ \midrule
        Local LR & \tt 1e-2 &\tt  1e-2 & \tt 1e-3 & \tt 1e-3 & \tt 1e-3  \\ \midrule
        Local LR (FedProx) & \tt 1e-3 &\tt  1e-3 & \tt 1e-3 &\tt  1e-2 &\tt  1e-2  \\ \midrule
        Local LR (AdaptFL) & \multicolumn{5}{c}{1e-1}  \\ \midrule
        Server LR &\tt  1 &\tt  5e-1 &\tt  1e-1 & \tt 5e-1 &\tt  5e-1  \\ \midrule
        Server LR (AdaptFL) & \multicolumn{5}{c}{\tt 1e-1}  \\ \midrule
        Cov. Param ($\rho$) & \multicolumn{5}{c}{\tt 0.4} \\\midrule 
        Proximal Param ($\lambda$) &\tt  1e-2 & \tt 1e-2 &\tt  1e-2 &\tt  1e-3 & \tt 1e-3\\\midrule
        Adaptivity ($\tau$) & \multicolumn{5}{c}{\tt 1e-2} \\\midrule 
        Sampler LR & \multicolumn{5}{c}{\tt 1e-1} \\ \midrule 
        Maximum Samples & \multicolumn{5}{c}{\tt 6}  \\ \midrule 
        Temperature & \multicolumn{5}{c}{$\frac{1}{|\mcD|}$} \\ \midrule
        Sampler Cycles & \multicolumn{5}{c}{\tt 5} \\ \midrule
        Samples per cycle & \multicolumn{5}{c}{\tt 2}  \\ \midrule
        Distill Optimizer & \multicolumn{5}{c}{Adam}  \\ \midrule
        Distill LR & \multicolumn{5}{c}{\tt 1e-4} \\ \midrule 
        Distill Epochs & \multicolumn{5}{c}{\tt 100} \\
        \bottomrule
    \end{tabular}
    }
     \caption{The tuned values of hyperparameters for the classification datasets, in the heterogeneous case $h>0$.} 
    \label{tab:tune_class_het}
\end{table*}

\begin{table*}[ht]
    \centering
     \scalebox{1.0}{
    \begin{tabular}{lccccc}
        \toprule
        \multirow{2}{*}{Hyperparameter}  & \multicolumn{5}{c}{Tuned Value}  \\
        \cmidrule(lr){2-6}
        & Air Quality & Bike & Wine Quality & Real Estate & Forest Fire  \\ 
        \midrule
        Optimizer & \multicolumn{5}{c}{Adam} \\ \midrule
        Optimizer (FedPA(S)) & SGDM & Adam & SGDM & SGDM & SGDM \\ \midrule 
        Optimizer (FedProx) & SGDM & Adam & Adam & Adam & SGDM \\ \midrule 
        Optimizer (AdaptFL) & SGDM & SGDM & SGD & SGD & SGD \\ \midrule 
        Local LR &\tt  1e-2 &\tt  1e-2 & \tt 1e-3 & \tt 1e-2 &\tt  1e-4 \\ \midrule
        Local LR (FedProx) &\tt  1e-2 &\tt  1e-2 & \tt 1e-2 &\tt  1e-1 &\tt  1e-1 \\ \midrule
        Local LR (AdaptFL) &\tt  1e-2 &\tt  1e-1 & \tt 1e-1 &\tt  1e-1 &\tt  1e-2 \\ \midrule
        Server LR &\tt  1e-1 &\tt  1e-2 &\tt  5e-1 &\tt  1 &\tt  1 \\ \midrule
        Server LR (AdaptFL) & \tt 1e-1 & \tt 1e-1 & \tt 1 & \tt 1e-1 &\tt  1 \\ \midrule
        Cov. Param ($\rho$) & \tt 0.4 &\tt  1.0 & \tt 0.9 & \tt 0.4 & \tt 0.4 \\ \midrule
        Proximal Param ($\lambda$) &\tt  1e-2 & \tt 1e-3 &\tt  1e-3 & \tt 1e-2 & \tt 1e-1 \\ \midrule
        Adaptivity ($\tau$) & \tt 1e-2 & \tt 1e-1 &\tt  1 &\tt  1e-2 & \tt 1e-2 \\ \midrule
        Sampler LR & \tt 1e-1 &\tt  2e-1 &\tt  2e-1 &\tt  2e-1 &\tt  1e-2 \\ \midrule 
        Maximum Samples & \tt 6 & \tt 4 &\tt  4 &\tt  6 &\tt  4 \\ \midrule 
        Temperature & \tt 1 & $\frac{1}{|\mcD|}$ & \tt 5e-2 & \tt 5e-1 &\tt  5e-1 \\ \midrule
        Sampler Cycles &\tt  5 &\tt  5 & \tt 4 &\tt  5 & \tt 2 \\ \midrule
        Samples per cycle & \tt 1 &\tt  2 & \tt 2 & \tt 2 & \tt 2 \\ \midrule
        Distill Optimizer & \multicolumn{5}{c}{Adam}  \\ \midrule
        Distill LR & \tt 1e-3 &\tt  1e-3 &\tt  5e-3 &\tt 5e-3 & \tt 5e-3 \\ \midrule 
        Distill Epochs & \multicolumn{5}{c}{\tt 100} \\
        \bottomrule
    \end{tabular}
    }
     \caption{The tuned values of hyperparameters for the regression datasets.} 
    \label{tab:tune_regr}
\end{table*}

\subsection{Heterogeneous Classification Dataset Construction}

The process for constructing a heterogeneous classification dataset is as follows: 

\begin{itemize}
    \item A parameter $h \in [0,1]$ is fixed.
    \item The dataset is sorted by class labels, and split evenly into shards for each client (the ``fully heterogeneous shards").
    \item A copy of the dataset is made and split such that each shard contains a roughly uniform class distribution for each client (the ``homogeneous shards").
    \item To form the final shard for a client, a fraction $h$ of each homogeneous shard is replaced with the data from the corresponding fully-heterogeneous shard.  
\end{itemize}

In this way, $h=0$ corresponds to the homogeneous data setting, while $h=1.0$ corresponds to a degenerate heterogeneous case (for class distributions in each client). 

\subsection{Regression Models}

A note on implementation for regression: the distilled student model is the same architecture as the client models, except in the last layer where it is made to output both a mean and an input-dependent variance. This network is trained to minimize the KL divergence between its output Gaussian distribution, and that of the teacher network. 

Additionally, for regression, only the baselines which provide an ensemble are evaluated (since the ensemble is used to predict the variance estimate of the output, $\sigma^2(x^*)$).

\subsection{Training $\beta$}

We choose $\beta$ in $\beta$-predBayes to minimize the negative log-likelihood on the server dataset $\mathcal{U}$. This is done using 10 epochs of gradient descent, using the Adam optimizer with a learning rate of \texttt{1e-2} (for all experiments).

\end{document}